\documentclass[11pt,letterpaper,twoside,reqno]{amsart}
\usepackage{amsmath,amsfonts,amsbsy,amsgen,amscd,mathrsfs,amssymb,amsthm,mathtools,tensor}
\usepackage{siunitx}
\usepackage{algorithm,algorithmic}
\usepackage[shortlabels]{enumitem}

\usepackage{authblk}
\usepackage[foot]{amsaddr}

\usepackage[toc, page]{appendix}

\usepackage[margin=1in]{geometry}
\usepackage[dvipsnames]{xcolor}
\definecolor{MyColor}{HTML}{0047AB}
\usepackage{fancyhdr}

\usepackage{hyperref}
\hypersetup{
colorlinks=true, 
linktoc=all,     
linkcolor=blue,  
}

\usepackage{etoolbox}
\makeatletter
\renewcommand{\@secnumfont}{\bfseries}
\makeatother
\patchcmd{\section}{\scshape}{\bfseries}{}{}
\patchcmd{\section}{\normalfont}{\normalfont\color{MyColor}}{}{}
\patchcmd{\subsection}{\normalfont}{\normalfont\color{MyColor}}{}{}
\makeatletter
\def\subsubsection{\@startsection{subsubsection}{3}
\z@{.5\linespacing\@plus.7\linespacing}{-.5em}
{\normalfont\bfseries}}
\makeatother

\usepackage{graphicx}
\usepackage{float}
\usepackage{caption}
\usepackage{subcaption}
\usepackage{tikz}

\usepackage{booktabs} 
\usepackage{multirow}
\usepackage{tabu} 

\usepackage[normalem]{ulem} 

\usepackage[color=yellow]{todonotes}

\newtheorem{theorem}{Theorem}[section]
\newtheorem{lemma}[theorem]{Lemma}
\newtheorem{definition}[theorem]{Definition}
\newtheorem{assumption}[theorem]{Assumption}
\newtheorem{proposition}[theorem]{Proposition}

\makeatletter
\@tfor\next:=abcdefghijklmnopqrstuvwxyzABCDEFGHIJKLMNOPQRSTUVWXYZ\do{
\def\command@factory#1{
\expandafter\def\csname b#1\endcsname{\mathbf{#1}}
\expandafter\def\csname fk#1\endcsname{\mathfrak{#1}}
\expandafter\def\csname bb#1\endcsname{\mathbb{#1}}
\expandafter\def\csname cl#1\endcsname{\mathcal{#1}}
\expandafter\def\csname bcl#1\endcsname{\mathbfcal{#1}}
}
\expandafter\command@factory\next
}

\newcommand{\tint}{\textstyle\int}

\makeatletter
\@namedef{subjclassname@2020}{\textup{2020} Mathematics Subject Classification}
\makeatother

\begin{document}

\title{A note on convergence of Wasserstein policy optimization}

\author{David \v{S}i\v{s}ka$^1$}
\email{d.siska@ed.ac.uk}
\author{Yufei Zhang$^2$}
\email{yufei.zhang@imperial.ac.uk}

\address{$^1$ School of Mathematics, University of Edinburgh, United Kingdom}
\address{$^2$ Department of Mathematics, Imperial College London, United Kingdom}

\begin{abstract}
Wasserstein Policy Optimization (WPO) is a recently proposed reinforcement learning algorithm that leverages Wasserstein gradient flows to optimize stochastic policies in continuous action spaces.
Despite its empirical success, the theoretical convergence properties of WPO in environments with continuous state and action spaces have yet to be fully established.
In this note, we argue that WPO within the framework of entropy-regularised Markov Decision Processes converges linearly.
This is done by leveraging recent advances in mean-field analysis for convergence of gradient flows using log-Sobole inequalities. 
Assuming existence of sufficiently regular solution to the gradient flow equation we demonstrate monotonic energy dissipation along the flow and establish a local log-Sobolev inequality.
Ultimately, these properties allow us to argue that the value function should converge linearly to the global optimum.
\end{abstract}

\maketitle 

\section{Introduction}

Reinforcement learning has achieved remarkable success in solving complex control problems with continuous action spaces.
A significant driver of this progress has been policy optimization, a family of methods that directly updates a policy's parameters via stochastic gradient  descent to minimize expected long-term costs.
Recently, Wasserstein Policy Optimization (WPO) was introduced as a novel actor-critic algorithm derived as an approximation to Wasserstein gradient flow over the space of all policies~\cite{pfau2025wasserstein}.
WPO bridges the gap between classic stochastic policy gradients and deterministic policy gradients: it exploits the gradient of the action-value function with respect to the action, yet it can be applied to arbitrary stochastic policies.

Despite its promising empirical performance, the theoretical convergence properties of WPO require a rigorous mathematical foundation.
The aim of this note is to answer the following question.

\vspace{1ex}
\begin{center}
{\em Can one expect WPO to converge and at what rate?}    
\end{center}
\vspace{1ex}

This note begins to answer by establishing the global linear convergence of WPO within the framework of entropy-regularised Markov Decision Processes (MDPs). 
Specifically, we consider an infinite horizon Markov decision model defined by a discrete and finite state space $S$, a continuous action space $A=\mathbb{R}^d$, a transition probability $P$, a bounded cost function $c$, and a discount factor $\gamma \in [0,1)$.

In this work, we postulate that the policy evolves according to a Wasserstein gradient flow driven by the flat derivative of the regularised value function.
We assume the gradient flow has solutions which are sufficiently well behaved so that, in particular, we can derive energy dissipation.
Our analysis extends techniques from the convex analysis of mean-field Langevin dynamics \cite{chizat2018global, chizat2018global, mei2018mean, hu2021mean, chizat2021convergence, nitanda2022convex, chizat2022mean}. We prove that energy dissipates monotonically along the flow. Furthermore, by demonstrating that a local log-Sobolev inequality holds along the gradient flow, we establish that the value function converges exponentially fast to the global optimal value.
While we carried out the analysis for a continuous-time gradient it should also be possible to establish convergence for discrete stepping schemes using the techniques proposed in~\cite{lascu2024linear}.
We believe that key ideas are easier to follow for the continuous time gradient flow.

The remainder of this note is organized as follows. Section \ref{sec:main_result} formally introduces the problem formulation and key tools, defining the entropy-regularised MDPs for continuous state and action spaces. Section 2 formulates the Wasserstein gradient flow continuity equation and presents our main theoretical results regarding its convergence analysis. Finally, Appendix A provides classical results and Bellman equations for entropy-regularised MDPs utilized throughout our proofs.

\section{Problem formulation and key tools} 
\label{sec:main_result}

\subsection{Entropy-regularised MDPs}
\label{sec:main_results_entropy_reg_MDPs}

In this section, we formulate the entropy-regularised MDPs with continuous state and action spaces. 
Let $S$ and $A$ be Polish spaces, $P\in \clP(S|S\times A)$, $c\in B_b(S\times A)$  and   $\gamma\in [0,1)$. 
The five-tuple  $(S,A,P,c,\gamma)$ determines an infinite horizon Markov decision model, where $S$ and $A$ represent the state and action spaces, respectively, $P$ represents the transition probability,  $c$ represents the cost function and $\gamma$ represents the discount factor.  
Let  $\Pi=\{\pi=\{\pi_n\}_{n\in \bbN_0}: \pi_n\in \clP(A|H_n)\}$ denote the set of (possibly non-Markovian) stochastic policies, where for each $n\in \bbN_0$, $H_n\coloneqq (S\times A)^{n}\times S$ is the space of admissible histories. 

Let $(\Omega:=(S\times A)^{\bbN_0},\mathcal{F})$ denote the canonical sample space, where $\clF=\clB(\Omega)$ is the corresponding Borel sigma-algebra. Elements of $\Omega$ are of the form $(s_0,a_0,s_1,a_1,\ldots)$ with $s_n\in S$ and $a_n\in A$ denoting the projections and called the state and action variables, at time $n\in \mathbb{N}_0$, respectively. By \cite[Proposition 7.28]{bertsekas2004stochastic}, for any given initial distribution $\rho\in\clP(S)$ and policy $\pi\in \Pi$, there exists a unique product probability measure $\bbP^{\pi}_{\rho}$ on $(\Omega,\mathcal{F})$ with expectation denoted $\mathbb{E}^{\pi}_{\rho}$ such that for all $n\in \bbN_0$, $B\in \clB(S)$ and $C\in  \clB(A)$, $\bbP_{\rho}^{\pi}(s_0\in B)=\rho(B)$ and 
\begin{equation} \label{eq:markov_like} 
\bbP_{\rho}^{\pi}(a_n\in C|h_n)=\pi_n(C|h_n),
\quad 
\bbP_{\rho}^{\pi}(s_{n+1}\in B|h_n, a_n)=P(B|s_n,a_n)\,, 
\end{equation}
where $h_n=(s_0,a_0,\ldots, s_{n-1}, a_{n-1}, s_{n})\in H_n$. In particular, if $\pi$ is a  Markov stochastic policy (i.e., $\pi_n\in \clP(A|S)$ for all $n\in \bbN_0$), then $\{s_n\}_{n\in \bbN_0}$ is a Markov process with kernel $\{P_{\pi,n}\}_{n\in \bbN_0}\in \clP(S|S)$ given by
\[
P_{\pi,n}(ds'|s)=\int_{A}P(ds'|s,a)\pi_n(da|s),
\quad \forall s\in S, n\in \bbN_0\,.
\]
For $s\in S$, we denote $\mathbb{E}^{\pi}_{s}=\mathbb{E}^{\pi}_{\delta_s}$, where $\delta_s\in \mathcal{P}(S)$ denotes the Dirac measure at $s\in S$.

Let $\mu \in \mathcal \clP(A)$ denote a reference measure and $\tau\in (0,\infty)$ denote a regularisation parameter. For each $\pi=\{\pi_n\}_{n\in \bbN_0} \in \Pi$ and $s\in S$, define the following regularised value function:
\begin{equation}
\label{eq:V_pi_tau}
V^{\pi}_{\tau}(s)= \bbE_{s}^{\pi}\left[\sum_{n=0}^\infty\gamma^n \Big(c(s_n,a_n) + \tau \operatorname{KL}(\pi_n(\cdot|h_n)|\mu)\Big)\right]\in \bbR\cup \{\infty\}\,,
\end{equation}
which may be infinite if  $\pi_n\not \in \clP_\mu(A|S)$ for some $n\in \mathbb{N}_0$, or if $\bbE_{s}^{\pi}\left[ \sum_{n=0}^\infty\gamma^n   \operatorname{KL}(\pi_n(\cdot|h_n)|\mu)\right]$ diverges. Since $c$ is bounded and $H_n\ni h_n\mapsto \operatorname{KL}(\pi_n(\cdot|h_n)|\mu)\in [0,\infty]$ is non-negative and measurable, $V^{\pi}_\tau: S\rightarrow \bbR\cup \{\infty\}$ is a well-defined measurable function. We define the optimal value function $V^*_{\tau}: S\rightarrow \bbR\cup \{\infty\}$ by
\begin{equation}
\label{eq:optimal_value}
V^*_{\tau}(s)=\inf_{\pi \in \Pi}V^{\pi}_{\tau}(s),
\quad \forall s\in S\,,
\end{equation}
and refer to  $\pi^* \in \Pi$ as an optimal policy if  $V^{\pi^*}_{\tau}(s)=V^*_{\tau}(s)$, for all  $s\in S$. 

One can prove that $V^*_{\tau}$ satisfies a dynamic programming principle (see Theorem~\ref{thm:DPP} for a  precise statement), which implies that $V^*_{\tau}\in B_b(S)$ and for all $s\in S$, 
\[
V^{\ast}_{\tau}(s)=-\tau\ln\int_{A}\exp\left(-
\frac{1}{\tau}Q^{\ast}_{\tau}(s,a)\right)\mu(da),
\]
where $Q^*_{\tau}\in B_b(S\times A)$ is defined by  
\[
Q^{*}_{\tau}(s,a)=c(s,a)+\gamma\int_S V_{\tau}^{*}(s')P(ds'|s,a)\,,
\quad \forall (s,a)\in S\times A\,.
\]
Moreover, there is an optimal policy $\pi^*_{\tau} \in \clP_{\mu}(A|S)$  given by
\begin{equation}
\label{eq:optimal_policy}
\pi^*_{\tau}(da|s) = \exp\left(-\frac{1}{\tau }(Q^{\ast}_{\tau}(s,a)-V^{\ast}_{\tau}(s))\right)\mu(da)\,,
\quad \forall s\in S.
\end{equation}
This suggests that, without loss of generality, it suffices to minimise \eqref{eq:V_pi_tau} over the class of stationary Markov policies that are equivalent to the reference measure $\mu$.

\begin{definition}
Let $\Pi_{\mu}$ denote the class of policies $\pi =\{\pi_n\}_{n\in \mathbb{N}_0} \in  \Pi $ such that $\pi_n\in \clP_\mu(A|S)$ for all $n\in \mathbb{N}_0$, and for which there exists $f\in B_b(S\times A)$ such that $\pi_n (da|s) =\frac{\exp\left(f(s,a)\right)}{ \int_{A}  \exp\left(f(s,a)\right) \mu(da)} \mu(da)$ for all $s\in S$ and $n\in \mathbb{N}_0$. In the sequel, we identify $ \Pi_\mu$ with the set  $\{\boldsymbol{\pi}(f)\mid f\in B_b(S\times A)\}\subset \clP_{\mu}(A|S)$, where $ \boldsymbol{\pi} :B_b(S\times A) \to \clP_{\mu}(A|S)$ is defined by 
\begin{equation}
\label{eq:pi_f_mu}
\boldsymbol{\pi}(f)(da|s)= \frac{e^{f(s,a)}}{\int_A e^{f(s,a')}\mu(d a')}\mu(d a),
\quad \forall f\in B_b(S\times A) \,.    
\end{equation}
\end{definition}

For each $\pi \in \Pi_{\mu} $, we define the $Q$-function $Q^{\pi}_{\tau}\in B_b(S\times A)$ by 
\begin{equation}\label{def:Q_fn}
Q^{\pi}_{\tau}(s,a)=c(s,a)+\gamma\int_S V_{\tau}^{\pi}(s')P(ds'|s,a)\,.
\end{equation}
Then due to the on-policy Bellman equation (see Lemma~\ref{lem:on_policy}), for all $\pi \in \Pi_{\mu} $ and $s\in S$,
\begin{equation}\label{eq:on_policy}
V^{\pi}_{\tau}(s)=\int_{A}\left(Q_\tau^\pi(s,a)+\tau \ln \frac{\mathrm{d} \pi}{\mathrm{d} \mu}(a|s)\right)\pi(da|s)\,.
\end{equation}
For each $\pi\in \clP(A|S)$, we define the occupancy kernel $d^{\pi}\in\clP(S|S)$ by
\begin{equation}
\label{eq:occupancy_s}
d^{\pi}(ds'|s)=(1-\gamma)\sum_{n=0}^{\infty}\gamma^nP^n_{\pi}(ds'|s)\,,
\end{equation}
where $P^n_{\pi}$ is the $n$-times  product of the kernel $P_{\pi}$ with $P^0_{\pi}(ds'|s)\coloneqq \delta_s(ds')$ and the convergence is understood in $b\clM(S|S)$. For a given initial distribution $\rho\in \clP(S)$, we define
\begin{equation}
\label{eq:occupancy_rho}
V^{\pi}_{\tau}(\rho)=\int_{S}V^{\pi}_{\tau}(s) \rho(ds) \quad \textnormal{and} \quad 
d^{\pi}_{\rho}(ds)=\int_{S}d^{\pi}(ds|s')\rho(ds')\,.
\end{equation}

The flat derivative 
is given by
\begin{equation}
\label{eq:delta_V_delta_pi-intro}
\frac{\delta V^{\pi}_{\tau}(\rho)}{\delta \pi}\bigg|_{\nu}(s,a)= \left(Q^{\pi}_\tau (s,a) + \tau \ln\frac{\mathrm{d}  \pi}{\mathrm{d}  \mu} (s,a) -V^{\pi}_\tau(s)\right)\frac{
\mathrm{d}
d^{\pi}_{\rho}}{\mathrm{d} \nu}(s)\,,
\end{equation}
where $d^{\pi}_{\rho}\in \clP(S)$ is the occupancy measure associated with $\pi$.
The flat derivative \eqref{eq:delta_V_delta_pi-intro}
generalises the notation of the flat derivative applied to probability measures to encompass probability transition kernels.

\section{Wasserstein Gradient Flow}
Throughout this section, 
we assume the action space $A = \bbR^d$
for some $d\in \bbN$ and $S$ is discrete and finite. 
Let $\lambda$ denote the Lebesgue measure on $A = \mathbb R^d$.
We start by postulating that starting with the Markov policy $\pi^0 \in \Pi_\mu$ the policies will evolve according to the continuity equation
\begin{equation}
\partial_t \tfrac{\mathrm d\pi_t}{\mathrm d\lambda} = \nabla_a \cdot (E_t \tfrac{\mathrm d\pi_t}{\mathrm d\lambda}) \,,\,\,\, t \in (0,\infty)\,,\,\,\, \tfrac{\mathrm d\pi_0}{\mathrm d\lambda} = \tfrac{\mathrm d\pi_0}{\mathrm d\mu} \tfrac{\mathrm d\mu}{\mathrm d\lambda}  = \tfrac{\mathrm d\pi^0}{\mathrm d\mu} e^{-U}\,\,\text{given.}
\end{equation}
We will abuse notation and for $\pi, \pi' \in \Pi_\mu$ and $s\in S$ write 
\begin{equation}
\operatorname{KL}(\pi|\pi')(s) := \operatorname{KL}(\pi(\cdot|s)|\pi'(\cdot|s))\,.
\end{equation}
Moreover, as $\tau > 0$ is fixed we shall drop it from various subscripts, so that, in particular, instead of $V^\pi_\tau$ we will write $V^\pi$.

Heuristically, with~\eqref{eq:delta_V_delta_pi-intro} and chain rule, we have
\begin{equation}
\begin{split}
\partial_t V^{\pi_t}
& = \frac{1}{1-\gamma}\int_S \int_A \frac{\delta V^{\pi_t}}{\delta \pi}(s,a) \partial_t \pi_t(da|s) \,d^{\pi_t}_\rho(ds)\\
& = \frac{1}{1-\gamma}\int_S \int_A \frac{\delta V^{\pi_t}}{\delta \pi}(s,a) \nabla_a (E_t \tfrac{\mathrm d\pi_t}{\mathrm d\lambda})(s,a)\lambda(da) \,d^{\pi_t}_\rho(ds)\\ 
& = -\frac{1}{1-\gamma}\int_S \int_A \nabla_a \frac{\delta V^{\pi_t}}{\delta \pi}(s,a) E_t(s,a) \tfrac{\mathrm d\pi_t}{\mathrm d\lambda}(a|s)\lambda(da) \,d^{\pi_t}_\rho(ds)\,. 
\end{split}	
\end{equation} 
We thus see that choosing $E_t = \nabla_a \frac{\delta V^{\pi_t}}{\delta \pi}(s,a)$ leads to energy dissipation
and thus the continuity equation can be replaced with
\begin{equation}
\label{eq:wasserstein_for_measure}
\begin{split}
\partial_t \tfrac{\mathrm d\pi_t}{\mathrm d\lambda} 
= \nabla_a \cdot \Big(\tfrac{\mathrm d\pi_t}{\mathrm d\lambda} \nabla_a \tfrac{\delta V^{\pi_t}}{\delta \pi} \Big) 
= \nabla_a \cdot \Big(\tfrac{\mathrm d\pi_t}{\mathrm d\lambda} \nabla_a Q^{\pi_t} + \tfrac{\mathrm d\pi_t}{\mathrm d\lambda} \tau \nabla_a \ln \tfrac{\textrm{d}\pi_t}{\textrm{d}\mu} \Big)\,. 
\end{split}
\end{equation}
for $t \in (0,\infty)$ with $\tfrac{\mathrm d\pi_0}{\mathrm d\lambda} = \tfrac{\mathrm d\pi_0}{\mathrm d\mu}e^{-U}$.	
Hence
\begin{equation}
\label{eq:wasserstein_for_density}
\begin{split}
\partial_t \tfrac{\mathrm d\pi_t}{\mathrm d\lambda} 
= \nabla_a \cdot \Big(\tfrac{\mathrm d\pi_t}{\mathrm d\lambda} \nabla_a Q^{\pi_t} + \tau \tfrac{\mathrm d\pi_t}{\mathrm d\lambda} \nabla_a U \Big) + \tau \Delta_a \tfrac{\mathrm d\pi_t}{\mathrm d\lambda} \,. 
\end{split}
\end{equation}
This has the stochastic representation 
\begin{equation}
\left\{
\begin{split}
d\alpha_t(s) & = - \big(\nabla_a Q^{\pi_t} + \tau\nabla U\big)(\alpha_t, s)	+ \sqrt{2\tau}dB_t(s)\,,\,\,t\geq 0\,,\,\, \alpha_0(s) \sim \tfrac{\mathrm d\pi^0}{\mathrm d\lambda} = \tfrac{\mathrm d\pi^0}{\mathrm d\mu}e^{-U}\,,  \\
\pi_t(\cdot|s) & = \text{Law}(\alpha_t(s))\,.
\end{split}	
\right.
\end{equation}

We will now extend a method applicable to minimisation of convex functions regularised with entropy presented in~\cite{nitanda2022convex}.
The analysis is more involved than the 
static minimisation problem studied in \cite{nitanda2022convex,chizat2022mean}.
In particular, 
the coefficient $ Q^{\pi}_\tau$ of   \eqref{eq:wasserstein_for_measure} 
may   be non-differentiable 
and unbounded with respect to $\pi$,
due to    the  $\rm KL$-divergence and the positive discount factor $\gamma$ (see \eqref{def:Q_fn}).
Moreover, 
as a given policy may induce a state  distribution that is a different from the optimal one, 
one has to control  such a distribution shift throughout the flow. 

\subsection{Convergence analysis }

We define the proximal policy 
\begin{equation}
\label{eq:proximal_policy}
\Phi[\pi'](da|s) := \underset{m \in \mathcal P(A)}{\operatorname{arg\,min}} \bigg(\int_A Q^{\pi'}(s,a) \,m(da) + \tau \operatorname{KL}(m|\mu) \bigg)\,.
\end{equation}  
What we're minimising over is exactly the right-hand-side of the policy Bellman equation, see Lemma~\ref{lem:on_policy} (with policy $\pi'$).
It is exactly the minimisation step one would be doing in the policy iteration algorithm.

\begin{proposition}[Entropy sandwich]
\label{propn. entropy sandwich}
Let $\pi'\in \Pi_\mu$.
Then for $\Phi[\pi']$ given by~\eqref{eq:proximal_policy} we have 
\begin{equation}
\frac{
\tau}{1-\gamma}\int_S \operatorname{KL}(\pi'|\pi^\ast)(s) d^{\pi'}_\rho(ds)	
= V^{\pi'}(\rho) - V^{\pi^\ast}(\rho) 
\leq \frac{\tau}{1-\gamma}\int_S  \operatorname{KL}(\pi'|\Phi[\pi'])(s)\, d^{\pi^\ast}_\rho(ds)\,. 
\end{equation} 

\end{proposition}
\begin{proof}

From~\cite[Lemma 1.4.3]{dupuis1997weak} we know that
\begin{equation}
\label{eq:formula_for_proximal_policy}
\Phi[\pi'](da|s) =  \frac{1}{Z_{\pi'}(s)} \exp\bigg(-\frac1\tau (Q^{\pi'}(s,a) - V^{\pi'}(s)) \bigg)\,\mu(da)\,,
\end{equation} 
where 
\begin{equation}
Z_{\pi'}(s) := \int_A \exp\bigg(-\frac1\tau (Q^{\pi'}(s,a') - V^{\pi'}(s))\bigg) \mu(da')\,.	
\end{equation}

The flat derivative of the objective can be written in terms of the proximal policy step 
\begin{equation}
\label{eq:flat_derivative_differenlty}
\begin{split}
\tfrac{\delta V^{\pi'}}{\delta \pi} & = Q^{\pi'} - V^{\pi'} + \tau \ln \tfrac{\mathrm d \pi'}{\mathrm d \mu}	
= \tau \ln \tfrac{\mathrm d \pi'}{\mathrm d \mu} -\tau \ln \exp \big(-\tfrac1\tau (Q^{\pi'} - V^{\pi'})\big) + \tau \ln Z_{\pi'} - \tau \ln  Z_{\pi'} \\
& =  \tau \ln \tfrac{\mathrm d \pi'}{\mathrm d \mu} - \tau \ln \tfrac{\mathrm d \Phi[\pi']}{\mathrm d \mu} - \tau \ln  Z_{\pi'} 
= \tau \ln \tfrac{\mathrm d \pi'}{\mathrm d \Phi[\pi']} - \tau \ln Z_{\pi'}\,. 
\end{split}
\end{equation}
Let us now note that Lemma~\ref{lem:performance_diff} implies that for any $\pi,\pi'$ and any $\rho$ that  
\begin{equation}
\begin{split}
& V^{\pi}_\tau(\rho)  = V^{\pi'}_\tau(\rho) +  \frac{1}{1-\gamma}\int_S \bigg[\int_A \frac{\delta V^{\pi'}}{\delta \pi}(s,a) (\pi-\pi')(da|s) + \tau    \operatorname{KL}(\pi|\pi')(s) \bigg]d^{\pi}_\rho(ds)\\
& = V^{\pi'}_\tau(\rho) +  \frac{1}{1-\gamma}\int_S \bigg[\int_A \bigg(Q^{\pi'}_\tau (s,a) + \tau \ln\frac{\mathrm{d}\pi'}{\mathrm{d}  \mu} (s,a) -V^{\pi'}_\tau(s)\bigg) (\pi-\pi')(da|s) + \tau    \operatorname{KL}(\pi|\pi')(s) \bigg]d^{\pi}_\rho(ds)\\
& = V^{\pi'}_\tau(\rho) +  \frac{1}{1-\gamma}\int_S \bigg[\int_A \bigg(Q^{\pi'}_\tau (s,a) + \tau \ln\frac{\mathrm{d}\pi'}{\mathrm{d}  \mu} (s,a)\bigg)  (\pi-\pi')(da|s) + \tau    \operatorname{KL}(\pi|\pi')(s) \bigg]d^{\pi}_\rho(ds)\\
& = V^{\pi'}_\tau(\rho) +  \frac{1}{1-\gamma}\int_S \bigg[\int_A \bigg(Q^{\pi'}_\tau (s,a) + \tau \ln\frac{\mathrm{d}\pi'}{\mathrm{d}  \mu} (s,a)\bigg)  (\pi-\pi')(da|s) + \int_A \tau \ln \frac{\mathrm d\pi}{\mathrm d\pi'}(s,a)\pi(da|s)\bigg]d^{\pi}_\rho(ds)\\
& = V^{\pi'}_\tau(\rho) +  \frac{1}{1-\gamma}\int_S \bigg[\int_A Q^{\pi'}_\tau (s,a)  (\pi-\pi')(da|s) + \tau \operatorname{KL}(\pi|\mu)(s)-\tau\int_A\ln \frac{\mathrm d\pi'}{\mathrm d \mu}(s,a)\pi'(da|s)\bigg]d^{\pi}_\rho(ds)\\
\end{split}
\end{equation}
If we minimise the term inside the integral $\int_S \cdots d^{\pi}_\rho(ds)$ for every $s\in S$ over $\pi$ and recall the proximal policy~\eqref{eq:proximal_policy} we get, for any $\pi,\pi'$ and any $\rho$ that  
\begin{equation}
\label{eq:from_perf_diff_2}
\begin{split}
& V^{\pi}_\tau(\rho) \geq V^{\pi'}_\tau(\rho) + \frac{1}{1-\gamma}\int_S \bigg[\int_A \frac{\delta V^{\pi'}}{\delta \pi}(s,a) (\Phi[\pi']-\pi')(da|s) + \tau \operatorname{KL}(\Phi[\pi']|\pi')(s) \bigg]d^{\pi}_\rho(ds).
\end{split}
\end{equation}
From~\eqref{eq:flat_derivative_differenlty} we have for any $\pi,\pi'$ that  
\begin{equation}
\begin{split}
& \int_A \frac{\delta V^{\pi'}}{\delta \pi}(s,a) (\Phi[\pi']-\pi')(da|s) + \tau \operatorname{KL} (\Phi[\pi']|\pi')(s)\\
& = \tau \int_A  \ln \frac{\mathrm d \pi'}{\mathrm d \Phi[\pi']} (s,a) (\Phi[\pi']-\pi')(da|s) + \tau  \int_A \ln \frac{\mathrm d \Phi[\pi']}{\mathrm d \pi'}(s,a)\Phi[\pi'](da|s)\\ 
& = - \tau \int_A  \ln \frac{\mathrm d \Phi[\pi']}{\mathrm d \pi'} (s,a) (\Phi[\pi']-\pi')(da|s) + \tau  \int_A \ln \frac{\mathrm d \Phi[\pi']}{\mathrm d \pi'}(s,a)\Phi[\pi'](da|s)\\ 
& = \tau \int_A  \ln \frac{\mathrm d \Phi[\pi']}{\mathrm d \pi'} (s,a) \pi'(da|s)\pi'(da|s) = - \tau \operatorname{KL}(\pi'|\Phi[\pi'])(s) \,. 
\end{split}	
\end{equation}
From this and~\eqref{eq:from_perf_diff_2} we get for any $\pi,\pi'$ and any $\rho$ that
\begin{equation}
\label{eq:from_perf_diff_3}
\begin{split}
& V^{\pi}_\tau(\rho) \geq V^{\pi'}_\tau(\rho) - \frac{\tau}{1-\gamma}\int_S  \operatorname{KL}(\pi'|\Phi[\pi'])(s)\, d^{\pi}_\rho(ds)\,.
\end{split}
\end{equation}
Hence for any $\pi,\pi'$ and any $\rho$ we have
\begin{equation}
\label{eq:V_in_terms_of_pi_1}
V^{\pi'}_\tau(\rho) - V^{\pi}_\tau(\rho) \leq \frac{\tau}{1-\gamma}\int_S  \operatorname{KL}(\pi'|\Phi[\pi'])(s)\, d^{\pi}_\rho(ds)\,.	
\end{equation}

From Theorem~\ref{thm:DPP} we know that for all $s,a$ we have
\begin{equation}
\frac{\delta V^{\pi^\ast}}{\delta \pi} = Q^{\pi^\ast} - V^{\pi^\ast} + \tau \ln \frac{\mathrm d\pi^\ast}{\mathrm d \mu} = 0\,. 
\end{equation} 
From this and Lemma~\ref{lem:performance_diff} we thus get for all $\pi'$ and $\rho$ that
\begin{equation}
V^{\pi'}(\rho) - V^{\pi^\ast}(\rho) = \frac{
\tau}{1-\gamma}\int_S \operatorname{KL}(\pi'|\pi^\ast)(s) d^{\pi'}_\rho(ds)\,.
\end{equation}
This and~\eqref{eq:V_in_terms_of_pi_1} with $\pi = \pi^\ast$ thus lead to 
\begin{equation}
\frac{
\tau}{1-\gamma}\int_S \operatorname{KL}(\pi'|\pi^\ast)(s) d^{\pi'}_\rho(ds)	
= V^{\pi'}(\rho) - V^{\pi^\ast}(\rho) 
\leq \frac{\tau}{1-\gamma}\int_S  \operatorname{KL}(\pi'|\Phi[\pi'])(s)\, d^{\pi^\ast}_\rho(ds)\,. 
\end{equation} 	
This concludes the proof.
\end{proof}

\begin{assumption}
\label{ass:nice_solution}
For each $s\in S$ the gradient flow~\eqref{eq:wasserstein_for_density} has a solution  $\tfrac{\mathrm d \pi}{\mathrm d\lambda}(\cdot|s) \in C^{2,1}(A\times (0,\infty))$.
The map $s\mapsto \tfrac{\mathrm d \pi}{\mathrm d\lambda}(\cdot|s)$ is measurable.
There are $K(s)>0$, $\delta(s)>0$ and for all $t > 0$ we have
\[
|\tfrac{\mathrm d \pi_t}{\mathrm d\lambda}(\cdot|s)| + |\partial_t \tfrac{\mathrm d \pi_t}{\mathrm d\lambda}(\cdot|s)| + |\nabla_a\tfrac{\mathrm d \pi_t}{\mathrm d\lambda}(\cdot|s) | \leq K(s) t^{(-d+2)/2}\exp(-\tfrac{1}{2t}\delta(s)|a|^2)\
\]
and $\sup_{t\geq 0}\int_S K(s) d^{\pi_t}_\rho(ds) < \infty$
\end{assumption}

The conditions in Assumption~\ref{ass:nice_solution} can most likely be obtained by refining arguments in~\cite{leahy2022convergence}.
These themselves rest on~\cite{bogachev2015fokker} and~\cite{lasota2013chaos}. 
The main difficulty is that the drift is only locally Lipschitz in policy. 

\begin{proposition}[Energy dissipation and log-Sobolev along the flow]
Let Assumption~\ref{ass:nice_solution} hold. 
Let the initial condition $\tfrac{\mathrm d\pi_0}{\mathrm d\lambda} = \tfrac{\mathrm d\pi_0}{\mathrm d\mu}e^{-U}$.
Then 
\begin{equation}
\label{eq:dV_dt}
\partial_t V^{\pi_t}(\rho) = -(1-\gamma)^{-1}\int_S \int_A \big|\nabla_a \tfrac{\delta V^{\pi_t}}{\delta \pi}(s,a) \big|^2 \tfrac{\mathrm d\pi_t}{\mathrm d\lambda}(a|s)\lambda(da)\,d^{\pi_t}_\rho(ds)\leq 0\,,\,\,t\geq 0\,.
\end{equation}
Moreover for each $t\geq 0$
\[
\|V^{\pi_t}\|_{B_b(S)} \leq \max\Big(\|V^{\pi_0}\|_{B_b(S)}, \|V^*\|_{B_b(S)}\Big)
\]
and
\[
\|Q^{\pi_t}\|_{B_b(S\times A)} \leq \|c\|_{B_b(S\times A)} + \gamma\max\Big(\|V^{\pi_0}\|_{B_b(S)}, \|V^*\|_{B_b(S)}\Big)\,.
\]
Finally, suppose $\mu(da) = e^{-U}\lambda(da)$ satisfies the log-Sobolev inequality.
Then following log-Sobolev inequality holds: there is $\alpha>0$ such that 
for all $t>0 $ and all $s\in S$, 
\begin{equation}
\label{eq:lsi_along_flow}
\operatorname{KL}(\pi_t|\pi_{\pi_t})(s) \leq \frac{1}{2\alpha} \int_A \bigg|\nabla_a \ln \frac{\mathrm d \pi_t}{\mathrm d \pi_{\pi_t}} \bigg|^2 \pi_t(da|s)\,.
\end{equation}
\end{proposition}
\begin{proof}
Under Assumption~\ref{ass:nice_solution}  the chain rule and integration-by-parts can be justified rigorously along the lines of~\cite[Theorem A.9]{leahy2022convergence} and yield
\begin{equation}
\begin{split}
\partial_t V^{\pi_t}
& = \tfrac{1}{1-\gamma}\tint_S \tint_A \frac{\delta V^{\pi_t}}{\delta \pi}(s,a) \partial_t \pi_t(da|s) \,d^{\pi_t}_\rho(ds)\\
& = \tfrac{1}{1-\gamma}\tint_S \tint_A \tfrac{\delta V^{\pi_t}}{\delta \pi}(s,a) \nabla_a \cdot ( \tfrac{\mathrm d\pi_t}{\mathrm d\lambda} \nabla_a \tfrac{\delta V^{\pi_t}}{\delta \pi})(s,a)\lambda(da) \,d^{\pi_t}_\rho(ds)\\ 
& = -\tfrac{1}{1-\gamma}\tint_S \tint_A \big|\nabla_a \frac{\delta V^{\pi_t}}{\delta \pi}(s,a)\big|^2  \tfrac{\mathrm d\pi_t}{\mathrm d\lambda}(a|s)\lambda(da) \,d^{\pi_t}_\rho(ds)\leq 0\,. 
\end{split}	
\end{equation} 
This shows~\eqref{eq:dV_dt}. 
As the value function is decreasing  for all $s\in S$ and $t\geq0$  we know that
\[
-\|V^\ast\|_{B_b(S)} \leq V^\ast(s) \leq V^{\pi_t}(s) \leq V^{\pi_0}(s) \leq \|V^{\pi_0}\|_{B_b(S)}\,.
\]
Hence, taking supremum over $s\in S$ we get for all $t\geq 0$ that
\[
-\|V^\ast\|_{B_b(S)} \leq \|V^{\pi_t}\|_{B_b(S)} \leq \|V^{\pi_0}\|_{B_b(S)}\,.
\]
This shows the second claim holds and~\eqref{def:Q_fn}, i.e. the definition of Q-function, leads to the third.
Finally, to prove the last claim, recall that~\eqref{eq:formula_for_proximal_policy} says
\[
\Phi[\pi_t](da|s) =  \tfrac{1}{Z_{\Phi[\pi_t]}(s)} \exp\big(-\tfrac1\tau (Q^{\pi_t}(s,a) - V^{\pi_t}(s)) \big)\,\mu(da)\,.
\]
Since we are assuming $\mu$ satisfies the log-Sobolev inequality and since we've shown $\sup_{t\geq0}\|Q^{\pi_t} - V^{\pi_t}\|_{B_b(S\times A)}<\infty$ we can use the Holley--Stroock criteria to conclude that log-Sobolev holds locally along the gradient flow~\eqref{eq:wasserstein_for_density} as claimed.  
\end{proof}

\begin{theorem}
Let Assumption~\ref{ass:nice_solution} hold.
Let $\bar \kappa := \sup_{s\in S} \frac{\mathrm d \rho}{\mathrm d d^{\pi^\ast}_\rho}(s)$ 
and 
$\underline \kappa := \inf_{s\in S} \frac{\mathrm d \rho}{\mathrm d d^{\pi^\ast}_\rho}(s)$.
Then 
\begin{equation}
\label{eq gronwall 2}
0\leq V^{\pi_t}(\rho) - V^{\pi^\ast}(\rho) 
\leq \bar \kappa \bigg(\int_S (V^{\pi_0}(s) - V^{\pi^\ast}(s))\,d^{\pi^\ast}_\rho(ds)\bigg)e^{-2\underline\kappa \alpha\tau t}\,.	
\end{equation}
\end{theorem}
Note that Lemma~\ref{lemma:bar_kappa} tells us $\bar \kappa < \infty$ and Lemma~\ref{lemma:underline_kappa} gives sufficient condition for $\underline \kappa > 0$.

\begin{proof}
Returning now to~\eqref{eq:wasserstein_for_density} and noticing that $\nabla_a V^{\pi_t} = 0$ we write 
\begin{equation}
\label{eq:wasserstein_for_density_2}
\partial_t \tfrac{\mathrm d\pi_t}{\mathrm d\lambda} = \nabla_a \cdot \Big(\tfrac{\mathrm d\pi_t}{\mathrm d\lambda} \tau \nabla_a \ln \exp\Big(\tfrac1\tau(Q^{p_t} - V^{p_t}+\tau U)\Big) + \tfrac{\mathrm d\pi_t}{\mathrm d\lambda} \tau \nabla_a \ln \tfrac{\mathrm d\pi_t}{\mathrm d\lambda}  \Big)\,.  
\end{equation}
Then~\eqref{eq:wasserstein_for_density_2},~\eqref{eq:formula_for_proximal_policy} and noting that $\nabla_a Z_{\pi_t} = 0$ leads to 
\begin{equation}
\label{eq:wasserstein_for_density_3}
\partial_t \tfrac{\mathrm d\pi_t}{\mathrm d\lambda} = \tau \nabla_a \cdot \Big(\tfrac{\mathrm d\pi_t}{\mathrm d\lambda} \nabla_a \big( -\ln \tfrac{\mathrm d\Phi[\pi_t]}{\mathrm d\lambda} + \ln \tfrac{\mathrm d\pi_t}{\mathrm d\lambda}\big)  \Big)
= \tau \nabla_a \cdot \Big(\tfrac{\mathrm d\pi_t}{\mathrm d\lambda}  \nabla_a \ln \tfrac{\mathrm d\pi_t}{\mathrm d\Phi[\pi_t]}\Big)
\,.  
\end{equation}
From this, the chain rule and integration by parts we then get
\begin{equation}
\begin{split}
\partial_t (V^{\pi_t} - V^{\pi^\ast}) 
= -\frac{\tau}{1-\gamma}\int_S \int_A \nabla_a \frac{\delta V^{\pi_t}}{\delta \pi}(s,a)   \nabla_a \ln \frac{\mathrm d \pi_t}{\mathrm d \Phi[\pi_t]}(s,a) \tfrac{\mathrm d \pi_t}{\mathrm d \lambda}(s,a) \lambda (da) \,d^{\pi_t}_\rho(ds)\,. 
\end{split}	
\end{equation} 
From this and~\eqref{eq:flat_derivative_differenlty} we get for any $\rho$ and any $t\geq 0$ that
\begin{equation}
\begin{split}
\partial_t (V^{\pi_t}(\rho) - V^{\pi^\ast}(\rho)) 
= -\frac{\tau^2}{1-\gamma}\int_S \int_A \bigg|\nabla_a  \ln \frac{\mathrm d \pi_t}{\mathrm d \Phi[\pi_t]}(s,a)\bigg|^2 \pi_t(da|s) \,d^{\pi_t}_\rho(ds)\,. 
\end{split}	
\end{equation} 
This with $\rho = \delta_s$ and with Lemma~\ref{lemma:from_occupancy_to_pointwise} (since the integrand is clearly non-positive) means that 
\begin{equation}
\label{eq:V_flow_pointwise}
\begin{split}
\partial_t (V^{\pi_t}(s) - V^{\pi^\ast}(s)) 
\leq -\tau^2 \int_A \bigg|\nabla_a  \ln \frac{\mathrm d \pi_t}{\mathrm d \Phi[\pi_t]}(s,a)\bigg|^2 \pi_t(da|s) \,. 
\end{split}	
\end{equation} 
We recall the local log-Sobolev inequality that holds along the flow~\eqref{eq:wasserstein_for_density}, see~\eqref{eq:lsi_along_flow}.
It's important to note that, unlike~\cite{nitanda2022convex}, we do not enforce a uniform log-Sobolev inequality for all $\pi'\in \clP(A|S)$.
This is because such uniformity often fails to hold, given that $Q^{\pi'}_\tau-V^{\pi'}_\tau$ is not uniformly bounded for all $\pi'\in \clP(A|S)$, primarily due to the inclusion of the KL divergence.
From~\eqref{eq:lsi_along_flow} and~\eqref{eq:V_flow_pointwise} we thus have for all $s\in S$ and $t\geq 0$ that
\begin{equation}
\label{eq:flow_after_log_sobolev}
\partial_t (V^{\pi_t}(s) - V^{\pi^\ast}(s)) 
\leq -2\alpha \tau^2 \operatorname{KL}(\pi_t|\Phi[\pi_t])(s)	\,.
\end{equation}
Integrating over $d^{\pi^\ast}_\rho$ and using~\eqref{eq:V_in_terms_of_pi_1} we get that 
\begin{equation}
\label{eq:flow_after_log_sobolev_2}
\partial_t \int_S (V^{\pi_t}(s) - V^{\pi^\ast}(s))\,d^{\pi^\ast}_\rho(ds) 
\leq 2\alpha \tau^2 \int_S \operatorname{KL}(\pi_t|\Phi[\pi_t])(s)\,d^{\pi^\ast}_\rho(ds)	
\leq -2 \alpha \tau (V^{\pi_t}(\rho) - V^{\pi^\ast}(\rho)) \,.
\end{equation}
We are nearly ready for an application of Gronwall's inequality but before we can do that, we observe that 
\begin{equation}
\label{eq:density_assumption}
\begin{split}
& (V^{\pi_t}(\rho) - V^{\pi^\ast}(\rho)) = \int_S (V^{\pi_t}(s)-V^{\pi^\ast}(s))\,\rho(ds)\\
& = \int_S (V^{\pi_t}(s)-V^{\pi^\ast}(s))\frac{\mathrm d \rho}{\mathrm d d^{\pi^\ast}_\rho}(s)\,d^{\pi^\ast}_\rho(ds) 
\geq \inf_{s\in S} \frac{\mathrm d \rho}{\mathrm d d^{\pi^\ast}_\rho}(s) \int_S (V^{\pi_t}(s)-V^{\pi^\ast}(s))\,d^{\pi^\ast}_\rho(ds)\,.
\end{split}
\end{equation}
This and~\eqref{eq:flow_after_log_sobolev_2} thus implies that 
\begin{equation}
\label{eq:flow_after_log_sobolev_3}
\partial_t \int_S (V^{\pi_t}(s) - V^{\pi^\ast}(s))\,d^{\pi^\ast}_\rho(ds) 
\leq -2 \underline \kappa \alpha \tau \int_S (V^{\pi_t}(s) - V^{\pi^\ast}(s))\,d^{\pi^\ast}_\rho(ds)  \,.
\end{equation}
Finally, with Gronwall's inequality we get that
\begin{equation}
\label{eq gronwall 1}
0\leq \int_S (V^{\pi_t}(s) - V^{\pi^\ast}(s))\,d^{\pi^\ast}_\rho(ds) \leq \bigg(\int_S (V^{\pi_0}(s) - V^{\pi^\ast}(s))\,d^{\pi^\ast}_\rho(ds)\bigg)e^{-2\underline\kappa \alpha\tau t}\,.	
\end{equation}
We also observe that 
\begin{equation}
\begin{split}
& (V^{\pi_t}(\rho) - V^{\pi^\ast}(\rho)) = \int_S (V^{\pi_t}(s)-V^{\pi^\ast}(s))\,\rho(ds)\\
& = \int_S (V^{\pi_t}(s)-V^{\pi^\ast}(s))\frac{\mathrm d \rho}{\mathrm d d^{\pi^\ast}_\rho}(s)\,d^{\pi^\ast}_\rho(ds) 
\leq \sup_{s\in S} \frac{\mathrm d \rho}{\mathrm d d^{\pi^\ast}_\rho}(s) \int_S (V^{\pi_t}(s)-V^{\pi^\ast}(s))\,d^{\pi^\ast}_\rho(ds)\,.
\end{split}
\end{equation}
Hence, from this and~\eqref{eq gronwall 1} we see that
\[
0\leq V^{\pi_t}(\rho) - V^{\pi^\ast}(\rho) 
\leq \bar \kappa \bigg(\int_S (V^{\pi_0}(s) - V^{\pi^\ast}(s))\,d^{\pi^\ast}_\rho(ds)\bigg)e^{-2\underline\kappa \alpha\tau t}
\]
which is the desired conclusion.
\end{proof}

\begin{lemma}
\label{lemma:bar_kappa}
Let $\bar{\kappa} := \sup_{s\in S} \frac{\mathrm d \rho}{\mathrm d d^{\pi^\ast}_\rho}(s)$.
Then $\bar \kappa \leq (1-\gamma)^{-1}$.
\end{lemma}
\begin{proof}
Observe that 
by \eqref{eq:occupancy_s}  and \eqref{eq:occupancy_rho},
for any $E\in \mathcal{B}(S)$ and $\pi\in \clP(A|S)$, 
\begin{equation}
\label{eq:kappa_overline_proof_1}
\begin{split}
d^{\pi}_\rho (E)
&= \int_{S}d^{\pi}(E|s')\rho(ds')
=\int_S (1-\gamma)\sum_{n=0}^{\infty}\gamma^nP^n_{\pi}(E|s') \rho(ds')
\\
&\ge \int_S (1-\gamma) P^0_{\pi}(E|s') \rho(ds')
=(1-\gamma)    \int_S \delta_{s'}(E) \rho(ds')
=  (1-\gamma) \rho(E)\,. 
\end{split}
\end{equation}
This implies that $\rho \ll  d^{\pi}_\rho$ and  
$\frac{\rm d \rho}{\rm d d^{\pi}_\rho} \le \frac{1}{1-\gamma}$ for $d^{\pi}_\rho  $ a.s. Indeed, suppose that there exists $E\in \mathcal{B}(S)$ such that 
$d^{\pi}_\rho (E)>0$ and 
$\frac{\rm d \rho}{\rm d d^{\pi}_\rho}> \tfrac{1}{1-\gamma} $ on $E$.
Then 
\[
\rho (E)=\int_S \frac{\rm d \rho}{\rm d d^{\pi}_\rho} (s) d^{\pi}_\rho(ds) >\frac{1}{1-\gamma }d^{\pi}_\rho (E)\,.
\]
But that is $
d^{\pi}_\rho (E) < (1-\gamma)\rho(E)
$
which contradicts~\eqref{eq:kappa_overline_proof_1}. 
This proves that $\bar{\kappa}\le \frac{1}{1-\gamma}$.
\end{proof}

\begin{assumption}
\label{ass:P_rho_reg}
There exists $K:S\times A\to [0,\infty)$ s.t. $\forall s',s,a$ we have $\frac{\mathrm d P(s|s',a)}{\mathrm d \rho(s)} \leq K(s',a)$ and $\int_S \int_A K(s',a) \pi^\ast (da|s')\,\rho(ds') < \infty$.
\end{assumption}

\begin{lemma}
\label{lemma:underline_kappa}
Let Assumption~\ref{ass:P_rho_reg} hold.
Let $\underline \kappa := \inf_{s\in S} \frac{\mathrm d \rho}{\mathrm d d^{\pi^\ast}_\rho}(s)$. 
Then $\underline \kappa > 0$.
\end{lemma}
\begin{proof}
Similarly to the proof of Lemma~\ref{lemma:bar_kappa} it is sufficient to show that there is $c_0 > 0$ such that 
\[
\rho (E)\ge c_0  d^{\pi^\ast}_\rho(E), \quad \forall E\in\mathcal {B}(S)\,.
\] 
Noting that since $P_\pi(E|s') \leq 1$ we have $P^n_\pi(E|s') \leq P_\pi(E|s')$ for $n\geq 1$ and so
\begin{equation}
\label{eq:density_estimates_2}
\begin{split}
d^{\pi}_\rho (E)
&= \int_{S}d^{\pi}(E|s')\rho(ds')
=\int_S (1-\gamma)\sum_{n=0}^{\infty}\gamma^nP^n_{\pi}(E|s') \rho(ds')
\\
& \leq \int_S (1-\gamma) P^0_{\pi}(E|s') \rho(ds')
+ \int_S (1-\gamma) P_\pi(E|s')\sum_{n=1}^\infty \gamma^n \,\rho(ds')\\
& \leq (1-\gamma)\rho(E)+\gamma \int_S P_\pi(E|s')\rho(ds')\\
& = (1-\gamma)\rho(E)+\gamma \int_S \int_A P(E|s',a)\,\pi(da|s')\rho(ds').
\end{split}
\end{equation}
Moreover, due to Assumption~\ref{ass:P_rho_reg}, we have
\[
P(E|s',a) = \int_E \frac{\mathrm d P(s|s',a)}{\mathrm d \rho(s)} \,\rho(ds) \leq \int_E K(s',a) \rho(ds) = \rho(E)K(s',a)
\]
and so $d^{\pi^\ast}_\rho (E) \leq c_0\rho(E)$
with 
\[
c_0 := 1-\gamma + \int_S \int_A K(s',a) \pi^\ast (da|s')\,\rho(ds')\,,
\]
which completes the proof.
\end{proof}

\section{Conclusion}
We have demonstrated that if the gradient flow~\eqref{eq:wasserstein_for_measure} has solutions postulated in Assumption~\ref{ass:nice_solution} then WPO can be expected to converge linearly.

\appendix 

\section{Classical results for entropy regularized MDPs}

The proofs of the Theorem~\ref{thm:DPP} and Lemmas~\ref{lem:on_policy} and~\ref{lem:performance_diff} can be found in~\cite{kerimkulov2025fisher}.

\begin{theorem}
[Dynamic programming principle]
\label{thm:DPP}
Let $\tau>0$.
The optimal value function $V^*_{\tau}$ 
is the unique bounded solution of the following Bellman equation:
\[
V^*_{\tau}(s)=\inf_{m \in \clP(A)}\int_{A}\left(c(s,a)+\tau \ln \frac{\mathrm{d} m}{\mathrm{d} \mu}(a)+\gamma \int_{S}V^*_{\tau}(s')P(ds'|s,a)\right)m(da),\quad \forall s\in S,.
\]
Consequently, 
for all $s\in S$, 
\[
V^{\ast}_{\tau}(s)=-\tau\ln\int_{A}\exp\left(-
\frac{1}{\tau}Q^{\ast}_{\tau}(s,a)\right)\mu(da),
\]
where
$Q^*\in B_b(S\times A)$ is defined by  
\[
Q^{*}_{\tau}(s,a)=c(s,a)+\gamma\int_S V_{\tau}^{*}(s')P(ds'|s,a)\,,
\quad \forall (s,a)\in S\times A.
\]
Moreover, 
there is an optimal policy $\pi^*_{\tau} \in \clP_{\mu}(A|S)$  given by
\[
\pi^*_{\tau}(da|s) = \exp\left(-(Q^{\ast}_{\tau}(s,a)-V^{\ast}_{\tau}(s))/\tau\right)\mu(da)\,,
\quad \forall s\in S.
\]
\end{theorem}

\begin{lemma}\label{lem:on_policy} 
Let $\tau>0$ and $\pi \in \Pi_{\mu}.$ The value function $V^{\pi}_{\tau}$ 
is the unique bounded solution of the following Bellman equation:
\[
V^{\pi}_{\tau}(s)=\int_{A}\left(c(s,a)+\tau \ln \frac{\mathrm{d} \pi}{\mathrm{d} \mu}(a|s)+\gamma \int_{S}V^{\pi}_{\tau}(s')P(ds'|s,a)\right)\pi(da|s),\quad \forall s\in S\,.
\]
\end{lemma}

\begin{lemma}[Performance difference]
\label{lem:performance_diff}
For all $\rho \in \clP(S)$ and $\pi,\pi'\in \Pi_{\mu}$, 
\begin{align*}
&V^{\pi}_\tau(\rho)-V^{\pi'}_\tau(\rho) \\
&\quad = \frac{1}{1-\gamma}\int_S \bigg[\int_A\left(Q^{\pi'}_{\tau}(s,a)+\tau \ln \frac{\mathrm{d} \pi'}{\mathrm{d}\mu}(a|s)\right)(\pi-\pi')(da|s) + \tau    \operatorname{KL}(\pi(\cdot | s)|\pi'(\cdot | s)) \bigg]d^{\pi}_\rho(ds)\,.
\end{align*}
\end{lemma}

Define a duality pairing $\langle \cdot, \cdot\rangle_{\nu}: B_b(S\times A)\times b\clM(A|S)\rightarrow \mathbb{R}$  by
\[
\langle Z, m\rangle_{\nu} = \frac{1}{1-\gamma}\int_{S}\int_{A} Z(s,a)m(da|s)\nu(ds)\,,\quad (Z, m)\in B_b(S\times A)\times b\clM(A|S).
\]
The flat derivative of $V^{\cdot}_{\tau}(\rho)$ relative to the duality pairing $\langle \cdot, \cdot\rangle_{\nu}$ is a map  
$\frac{\delta V^{\cdot}_{\tau}(\rho)}{\delta \pi}\big|_{\nu}:  \Pi_{\mu}\to B_b(S\times A)$ such that for every $\pi,\pi'\in \Pi_\mu$
\begin{equation*}
\lim_{\varepsilon\searrow 0}\frac{V^{(1-\varepsilon)\pi+\varepsilon\pi'}_\tau(\rho) - V^\pi_\tau(\rho)}{h} = \left\langle \frac{\delta V^{\pi}_{\tau}(\rho)}{\delta \pi}\bigg|_{\nu}, \pi' - \pi  \right\rangle_\nu\,\,\,
\text{and}\,\,\,\left\langle \frac{\delta V^{\pi}_{\tau}(\rho)}{\delta \pi}\bigg|_{\nu}, \pi  \right \rangle_\nu = 0.    
\end{equation*}

\begin{lemma}
\label{lemma:from_occupancy_to_pointwise}
Let $F:S \to \mathbb R$ be such that $F\leq 0$. 
Then for any $\pi$ and any $s\in S$
\begin{equation}
\frac{1}{1-\gamma} \int_S F(s') \, d^{\pi}_s (ds') \leq F(s)\,.	
\end{equation}

\end{lemma}
\begin{proof}
From~\eqref{eq:occupancy_s} and the fact that $P^0_\pi(ds'|s) = \delta_s(ds')$ we have  for all $s\in S$ that
\begin{equation}
\begin{split}
\frac{1}{1-\gamma} \int_S F(s') \, d^{\pi}_s (ds') & = \int_S F(s') P^0_{\pi}(ds'|s) + \sum_{k=1}^\infty \int_S \gamma^k F(s') P^k_{\pi}(ds'|s)\\
& \leq \int_S F(s') \delta_s(ds') = F(s)\,.	
\end{split}
\end{equation} 
This concludes the proof.
\end{proof}

\bibliographystyle{siam}
\bibliography{bibliography}

\end{document}